\documentclass{article}

\usepackage{microtype}
\usepackage{graphicx}
\usepackage{subfigure}
\usepackage{booktabs} 

\usepackage[utf8]{inputenc}
\usepackage{amsmath}
\usepackage{amssymb}
\usepackage{amsthm}
\usepackage{xcolor}
\usepackage[normalem]{ulem}
\usepackage{wrapfig}

\usepackage{algorithm,algorithmic}

\usepackage[pdftex,colorlinks,linkcolor=blue,citecolor=blue,filecolor=blue,urlcolor=blue]{hyperref}
\usepackage[letterpaper,hmargin=1.0in,vmargin=1.0in]{geometry}
\usepackage[margin=20pt,font=small,labelfont=bf]{caption}
\usepackage{hyperref}
\usepackage{nicefrac}
\newtheorem{theo}{Theorem}[section]

\newtheorem{lemma}[theo]{Lemma}

\theoremstyle{remark}

\theoremstyle{definition}

\def \tr {\operatorname{tr}}

\title{COPT: Coordinated Optimal Transport for Graph Sketching}
\author{
  Yihe Dong\\
  Microsoft \\
  \and
  Will Sawin \\
  Department of Mathematics \\
  Columbia University \\
}

\begin{document}

\date{}
\maketitle

\begin{abstract}
We introduce COPT, a novel distance metric between graphs defined via an optimization routine, computing a coordinated pair of optimal transport maps simultaneously. This gives an unsupervised way to learn general-purpose graph representation, applicable to both graph sketching and graph comparison. COPT involves simultaneously optimizing dual transport plans, one between the vertices of two graphs, and another between graph signal probability distributions. We show theoretically that our method preserves important global structural information on graphs, in particular spectral information, and analyze connections to existing studies. Empirically, COPT outperforms state of the art methods in graph classification on both synthetic and real datasets.
\end{abstract}

\section{Introduction}
\label{sec-intro}
We introduce a new unsupervised method to measure the distance between a pair of graphs, and apply it to graph sketching. 
This distance is based on the general notion of optimal transport distance, which involves minimizing a loss function over transport plans between two distributions \cite{Kontorovich}. However, our distance is defined by minimizing a loss function over \emph{pairs} of simultaneous transport plans, one between the vertices of the two graph and one between distributions defined by the Laplacian spectra of the graphs. This allows us to compare, in a flexible way, large-scale spectral information between the two graphs. Thus, we call it Coordinated OPtimal Transport, or COPT. We show that COPT has desirable properties in theory, as well as empirically demonstrate its usefulness in graph sketching, retrieval, and summarization, on both synthetic and real world datasets. 

Constructing a distance metric between graphs and studying its applications come from a long, rich line of work, due to the ubiquity of graph-strcutured data and the importance of graph sketching and retrieval. We briefly highlight some recent developments in this field, while drawing more detailed connections throughout the text.

Sketching is often defined as choosing a sequence of combinatorial operations (e.g. edge contractions) that minimizes a measure of distance between the sketch and the original graph. For instance, \cite{LoukasVandergheynst} contracts edges in such a way as to preserve the Laplacian spectrum. \cite{HermdorffGunderson} removes edges and merges vertices in a way that minimizes the Frobenius norm of changes in the psuedoinverse of the Laplacian.

Sketching has been applied to a number of different graph problems. \cite{RonSafroBrandt}  used iterated graph sketching to find optimal orderings of the vertices of a graph. Building on this, \cite{ChenSafro} defined an efficiently computable notion of distance on graphs. \cite{LivneBrandt} used sketching to efficiently solve linear systems involving the graph Laplacian. Generalizing ideas from electrical engineering, \cite{DorflerBullo} defined a sketch as the Schur complement of the Laplacian with respect to a subset of the nodes. \cite{ShumanFarajiVandergheynst} generalized multiscale methods to graphs with extra structure on the nodes.  
We discuss further connections to related works in \S\ref{sec-prior-work}, as well as provide empirical comparisons to competitive baselines.

Our \textbf{main contributions} are 1) devising a coordinated optimal transport algorithm for computing graph distance; 2) applying COPT to graph sketching, obtaining small sketches that allow for improved graph classfication, insightful visualization, and high quality retrieval. 

This paper is outlined as follows: In \S\ref{s-general}, we review generalities on optimal transport methods for graph comparison and discuss prior work. In  \S\ref{s-COPT}, we define coordinated optimal transport distance and discuss its properties. In \S\ref{s-sketching}, we describe our approach to graph sketching. In \S\ref{s-experiments}, we discuss algorithm implementations and experimental results.


\section{Graph distances based on optimal transport on vertices}\label{s-general} 

In general, we would like a notion of ``distance" for graphs that satisfies the properties of a metric, and in particular is zero if and only if the two graphs are isomorphic. We would also like this distance to be reasonably computable in practice.

Because no simple complete invariant for graphs up to isomorphism is known, the most natural approach to define a distance for graphs up to isomorphism is to define the distance between graphs $X$ and $Y$ as a minimum, over bijections between the vertices of $X$ and the vertices of $Y$, of some quantity, which vanishes if and only if this bijection sends the edges of $X$ to the edges of $Y$. For instance, we could take the minimum over bijections of the cardinality of the symmetric difference of the edge set. However, there are some downsides to minimizing over permutations.

First, such a distance would be hard to compute, or even approximate, in practice, as it involves a complicated discrete optimization problem.

Second, such a distance would not be defined if our graphs $X$ and $Y$ have different numbers of vertices.

To solve these problems, we can define a graph distance as a minimization over transport plans. To define these, we first \textbf{fix some notation}. Let $X$ be a graph with $N$ vertices and $Y$ a graph with $M$ vertices. We will also use $X$ and $Y$ to denote the set of vertices of $X$ and $Y$ respectively. Optimal transport plans  are functions $P$ from $X \times Y$ to $\mathbb{R}\cup \{0\}$ such that $\sum_{x \in X} P (x,y) = N$ for all $y\in Y$, and $\sum_{y \in Y} P(x,y) = M$ for all $x\in X$. We will define distances as a minimum over transport plans $P$, so their formulations will be analogous to the optimal transport, or Wasserstein, distance, defined as \begin{equation}\label{eq-wasserstein-distance} W_p(X,Y) = \min_{\substack{  P: X\times Y \to R^{+} \\ \sum_{x \in X} P (x,y) = N \\ \sum_{y \in Y} P(x,y) = M}} \Bigl( \sum_{x\in X}\sum_{y\in Y} d(x,y)^p P(x,y) \Bigr)^{1/p}\end{equation} where $d(x,y)$ is a distance function between two points.

\noindent{\bf Gromov-Wasserstein distance.} However, before being able to apply \eqref{eq-wasserstein-distance}, there is no notion of $d(x,y)$ for $x,y$ vertices in two different graphs.

To fix this, M\'{e}moli proposed a notion of Gromov-Wassestein distance for graphs \cite{Memoli}, as  \begin{equation}\label{eq-gromov-Wasserstein} \Bigl(\min_{\substack{  P: X\times Y \to R^{+} \\ \sum_{x \in X} P (x,y) = N \\ \sum_{y \in Y} P(x,y) = M}} \sum_{\substack{  x_1,x_2\in X\\ y_1,y_2 \in Y}} (d_X(x_1,x_2) - d_Y(y_1,y_2))^p  P(x_1,y_1) P(x_2,y_2) \Bigr)^{1/p}. \end{equation} 
In other words, given the distance $d_X$ for two vertices in the same graph, defined as the minimum number of edges in a path connecting them, we have a natural notion of distance between two pairs $x_1,x_2$ and $y_2,y_2$ of vertices on the two different graphs as the difference between the distances of the individual vertices.

A generalization of this definition to arbitrary functions $L$, along with computational methods and applications, was provided by \cite{PeyreCuturiSolomon}, building on computational ideas of \cite{cuturi2013sinkhorn}. An application to word embeddings was given by \cite{Alvarez-MelisJakkola}. A similar approach, but based more closely on Gromov-Hausdorff distance, was due to Sturm \cite{Sturm}. 

\noindent{\bf Graph Optimal Transport.} The recently proposed GOT \cite{MGCP} graph distance uses optimal transport in a different way. This relies on a probability distribution $\mu^X $, the \emph{graph signal} of $X$ \cite{rueHeld, graphSignalDTRF}, over functions on the vertices of $X$. This distribution is a  multivariate Gaussian, with mean zero, whose variance-covariance matrix is a pseudo-inverse $L_X^{\dagger}$ of the Laplacian $L_X$. They then define, in the case $N=M$, a distance for graphs defined by optimal transport of these probability distributions. Let $T: \mathbb R^X \to \mathbb R^Y$ denote a transport plan and $\sigma: X \to Y$ a permutation, \cite{MGCP} defines a distance as
\begin{equation}\label{eq-GOT} W_2 (\mu^X, \mu^Y)^2 = \min_{\substack{  \sigma: X \to Y\\ \sigma  \textrm{ bijective}}} \inf_{ \substack{ T: \mathbb R^X \to \mathbb R^Y \\ T_{\#} \mu^X = \mu^Y }}  \int_{\mathbb R^X} \sum_{ x\in X} (f(x) - Tf( \sigma(x)))^2 d\mu^X (f)  \end{equation} 

\section{Coordinated optimal transport}\label{s-COPT}

Our definition of a new metric on graphs builds on \eqref{eq-GOT}, where we replace the permutation $\sigma$ with an optimal transport plan $P$. Thus, our definition involves two different optimal transport plans: $P, T$, hence named \emph{coordinated optimal transport}. We define our distance $\Delta(X,Y)$ by
\begin{equation}\label{eq-COPT-first} NM \Delta(X,Y)^2 =  \min_{\substack{  P: X\times Y \to R^{+} \\ \sum_{x \in X} P (x,y) = N \\ \sum_{y \in Y} P(x,y) = M}} \inf_{ \substack{ T: \mathbb R^X \to \mathbb R^Y \\ T_{\#} \mu^X = \mu^Y }}  \int_{\mathbb R^X} \sum_{ x\in X} \sum_{y \in Y}  (f(x) - Tf(y))^2 P(x,y)  d\mu^X (f)  .\end{equation} 
Again, we take $\mu^X$ to be a Gaussian with mean zero and variance-covariance matrix $L_X^\dagger$. In the special case that $N=M$ and $P$ is a permutation, this definition reduces to the definition in \cite{MGCP}, up to a normalization factor of $\sqrt{N}$. COPT is more general and can be used between graphs of different cardinalities and for sketching.

\subsection{Properties of COPT}

We give an analytic formula for computing COPT distance $\Delta(X,Y)$, and show $\Delta(X,Y)$ is a metric. See the supplementary material for full proofs.

\begin{lemma}\label{min-formula} Let $X$ and $Y$ be graphs with vertex sets of size $N$ and $M$ respectively. Then
\begin{equation}\label{eq-min-formula} \begin{split} & \inf_{ \substack{ T: \mathbb R^X \to \mathbb R^Y \\ T_{\#} \mu^X = \mu^Y }} \int_{\mathbb R^X} \sum_{ x\in X} \sum_{y \in Y}  (f(x) - Tf(y))^2 P(x,y)  d\mu^X (f)\\  =&  M \tr (L_X^\dagger) + N \tr (L_Y^{\dagger})  - 2\tr (   ((L_Y^{\dagger})^{1/2} P^T L_X^{\dagger} P (L_Y^{\dagger})^{1/2})^{1/2}) \end{split} \end{equation} where $P$ is the matrix with entries $P(x,y)$.
\end{lemma} 

\begin{proof}[Proof summary] We extend $\mu^X$ and $\mu^Y$ to distributions on the space of functions on $X \times Y$, in such a way that the infimum we are interested in is exactly the Wasserstein distance between these distributions. Because the extended distributions remain multivariate Gaussians, we can use the known Wasserstein distance formula for multivariate Gaussians \cite{wassGaussian}. By calculating the variance of these extended distributions, and using the cyclic permutation invariance of the traces of powers of a matrix, this reduces to our stated formula. \end{proof}

Using this analytic formula for the minimum over $T$, we can approximate the coordinated optimal transport distance by using gradient descent to handle the minimum only over $P$.

\begin{lemma}\label{COPT-metric} $\Delta(X,Y)$ is a metric on the set of isomorphism classes of finite graphs. \end{lemma}

\begin{proof}[Proof summary]  We check each axiom from the definition of a metric separately. For each of them, our strategy is based on the corresponding step in the proof that the Wasserstein distance is a metric. Whatever construction must be applied to the transport map or joint measure in the Wasserstein distance proof is applied to both $P$ and $T$ in our proof. For instance, to check the triangle inequality, we compose the transport maps $T$ and also compose $P$ by a matrix multiplication: Given $P: X \times Y \to \mathbb R$ and $Q: Y \times Z \to \mathbb R$, we take $\frac{1}{M}\sum_{y\in Y} P(x,y)Q(y,z): X \times Z \to \mathbb R $. The calculations in each step are similar to, but more intricate than, the calculations in the Wasserstein distance proof.\end{proof}


\subsection{Global information: spectral vs. metric}\label{comparison-GW}

We analyze the graph global structures that COPT preserves, and compare and contrast COPT with the state of the art Gromov-Wasserstein (GW) distance \cite{VayerChapelFlameryTavenardCourty, Memoli, PeyreCuturiSolomon}. The COPT metric and GW metric are both optimal transport metrics for graphs. The main difference is in what information about the graph they emphasize. The GW distance is defined in terms of the metric $d_X$, and so it measures primarily changes to the graph that change the distance function by a large amount, while COPT is defined in terms of $L_X^\dagger$, so it measures primarily changes to the graph that change the eigenvectors of the Laplacian with \emph{small} eigenvalue by a large amount.

To see the difference between these two concepts, consider a graph with two clusters. The distance between a point in the first cluster and a point in the second cluster is determined mainly by the \emph{length of the shortest path} between the clusters. Adding new paths between the clusters will not change the distance much, while lengthening all paths will change the distance drastically. On the other hand, the entries of the matrix $L_X^\dagger$ with row in the first cluster and column between the cluster is determined more by the \emph{number of paths} between these clusters. As we add more and more paths, these entries of $L_X^\dagger$ will get less and less negative,  up until the number of paths between the clusters is almost as large as the number of paths within a cluster. On the other hand, lengthening the paths will affect $L_X^\dagger$ less. 

To see the relationship between the graph Laplacian and counting short paths, it is convenient to use the formula: 
\[ L_X^\dagger = (D_X - A_X)^{-1} = D_X^{-1} +D_X^{-1} A_XD_X^{-1} + D_X^{-1} A_X D_X^{-1} A_X D_X^{-1} + \dots \] 
where $A_X$ is the adjacency matrix and $D_X$ is a diagonal matrix whose diagonal entries are the degrees of each vertex. Thus, each entry of the Laplacian pseudoinverse is a formal series counting paths, where for instance the entries of $ D_X^{-1} A_X D_X^{-1} A_X D_X^{-1} $ are a weighted count of paths of length two. (As long as $X$ is connected and not bipartite, this sum converges once we orthogonally project each term onto the complement of the all $1$s matrix, as doing this removes the influence of the all $1$s eigenvector, and then we can use the convergence of the geometric series.)


\section{COPT for graph sketching}\label{s-sketching}

{\bf Motivation.} Graph sketching replaces a graph with a structurally similar graph with a smaller number of vertices. Many sketching methods focus on preserving the spectrum of the graph, but the best similarity metric may depend on the task. 
Graph sketching techniques have wide applications, beyond what is discussed in \S\ref{sec-intro}, it has also been used to reduce computational load and memory footprint \cite{korenSketchCompute}, as part of graph convolution networks to learn a hierarchical scaling of graph representations and reduce overfitting \cite{simonskyGCN, bronsteinGCN, defferrardGCN, brunaGCN},
and as a key subroutine in graph partitioning \cite{lelandPartition, partitionKarypis, partitionKushnir, partitionDhillon}.

Using the COPT distance function between graphs, we define a method to sketch a graph by reducing it to a low-dimensional matrix, i.e. the sketched Laplacian. The sketched Laplacian preserves key spectral information about the graph. 
Given a graph $X$ on $N$ vertices and a target size $M$, we search for the graph $Y$ on $M$ vertices that minimizes our distance function $\Delta(X,Y)$. In theory, this graph would be the $M$-vertex graph that best approximates $X$, and therefore should share many of the same features (e.g. clusters or the lack thereof), but with fewer vertices. 

\noindent{\bf Method.} However, there are two problems with reducing to a smaller graph: 1) this is a discrete optimization problem, and continuous optimization problems are computationally simpler; 2) the number of isomorphism classes of graphs on a small vertex set is relatively small, so smaller graphs cannot preserve much information.

The solution to both these difficulties is to relax the question. Our distance function depends only on the Laplacian $L_Y$ of $Y$. Rather than finding the graph that minimizes the distance, we find the Laplacian $L_Y$ that minimizes the distance. We choose $L_Y$ subject to the conditions typical of a Laplacian matrix - it is symmetric, its off-diagonal entries are nonpositive, and its row and column sums vanish. Alternatively, we can view this as finding the weighted graph $Y$ whose distance to $X$, defined using the weighted Laplacian, is minimized.

Formally, the sketch of the graph $X$ is given by the $L_Y$ which attains the minimum
\begin{equation}\label{eq-sketch-min}  \min_{\substack{L_Y \in M_Y (\mathbb R) \\ (L_Y)_{y_1y_2} = (L_Y)_{y_2y_1} \\ (L_Y)_{y_1y_2} \leq 0 \textrm { if } y_1\neq y_2 \\ \sum_{y_2} L_{y_1y_2} =0 }} \min_{\substack{  P: X\times Y \to R^{+} \\ \sum_{x \in X} P (x,y) = N \\ \sum_{y \in Y} P(x,y) = M}} \left(  M \tr (L_X^\dagger) + N \tr (L_Y^{\dagger})  - 2\tr (   ((L_Y^{\dagger})^{1/2} P^T L_X^{\dagger} P (L_Y^{\dagger})^{1/2})^{1/2})  \right) \end{equation} 

In practice, we use a gradient descent algorithm on both $P$ and $L_Y$ simultaneously to find an approximate minimum.

\begin{algorithm}[tb]
   \caption{COPT graph sketching}
   \label{alg:copt}
\begin{algorithmic}
   \STATE{\bfseries Input:} Graph $X$ of size $N$, target sketch dimension $M$
   \STATE{\bfseries Initialize:} $L_X^{\dagger} \leftarrow$ inverse Laplacian of $X$
   \STATE{\bfseries Initialize:} $(L_Y)'$: the $M(M-1)/2$ strict upper triangular entries of $L_Y$, drawn from $\mathcal{N}(0, 1)$ 
   \STATE{\bfseries Initialize:} $P(x,y)$ for $x\in X, y\in Y$, sampled from Uniform$(1,2)$
   \FOR{$i=1$ {\bfseries to} n\_iter}
   \STATE{\bfseries Set} $P(x,y) = \text{abs}(P(x,y))$
   \STATE{\bfseries Normalize} $P(x,y)$ by 5 iterations of Sinkhorn-Knopp algorithm
   \STATE{\bfseries Ensure Laplacian properties:} for $y_1<y_2$, $(L_Y)_{y_2y_1} \leftarrow -(L_Y')^2$, $(L_Y)_{y_1y_2}, \leftarrow (L_Y)_{y_2y_1} $, $(L_Y)_{y_1y_1}, \leftarrow - \sum_{y_2 \neq y_1} L_{y_1y_2} $
   \STATE{\bfseries Minimize} objective Eq~\eqref{eq-COPT-first}: 
   \STATE{\bfseries I: } Compute gradient of Eq~\eqref{eq-min-formula} evaluated at $L_Y$ and $P(x,y)$ 
   \STATE{\bfseries II: } Update $L_Y'$,and $P(x,y)$ using gradient
   \ENDFOR
   \STATE{\bfseries Return:} $L_Y$: Laplacian of sketched graph; $P(x,y)$: the transport plan
\end{algorithmic}
\vspace{-2pt}
\end{algorithm}

\subsection{Implementation}

As summarized in Algorithm \ref{alg:copt}, the values of the transport plan $P$ are initially uniformly sampled from the interval $[1, 2]$. At the beginning of each iteration, we normalize $P$ so its row and column sums are equal, by using the Sinkhorn-Knopp algorithm \cite{cuturi2013sinkhorn}. This ensures that $P$ is a transport plan.

$L_Y'$ corresponds to the upper triangular part of $L_Y$, as that entirely determines the Laplacian. $L_Y'$ is initialized from the standard Gaussian. At the start of each iteration, $L_Y$ is obtained by taking its upper triangular part to be $-(L_Y')^2$, then symmetrized, and diagonal terms filled, to ensure it's a Laplacian matrix.

Gradient descent is used to minimize the analytic formulation Eq~\eqref{eq-min-formula}, where $L_Y'$ and $P$ are updated at each step, with the Adam optimizer \cite{adam2015} with a multistep learning rate scheduler that reduces the learning rate multiplicatively at regular intervals. When using COPT to determine the distance between two graphs with Laplacians $L_X$ and $L_Y$, only $P$ is optimized over in each iteration. Our implemention uses PyTorch and one P100 GPU, on a 2.60GHz six-core Intel CPU machine.


As sanity checks, we confirm that 1) the sketched Laplacian $L_Y$ converges to the original graph's Laplacian $L_X$ when the target sketch dimension is that of the original graph, and 2) the distance converges to 0 when $L_X$ and $L_Y$ are fixed to be equal. The effects of $P$ as a transport plan can be seen from the node labels in sketched graphs in Figure~\ref{fig:summary}. 




\subsection{Time complexity}

We estimate the coordinated optimal transport distance by a gradient descent algorithm. The time complexity is given by (number of iterations $\times$ time to calculate each iteration). We are unaware of a general method to estimate the number of iterations needed to converge (in practice $\sim$150 iterations suffice to sketch 50-node to 15-node graphs, and $\sim$1000 iterations to sketch 1000-node to 200-node graphs), so we focus on estimating the time per iteration, where the bottleneck is evaluating \[M \tr (L_X^\dagger) + N \tr (L_Y^{\dagger})  - 2\tr (   ((L_Y^{\dagger})^{1/2} P^T L_X^{\dagger} P (L_Y^{\dagger})^{1/2})^{1/2}) \] and its derivative with respect to $P$. This can be done in matrix multiplication time $O( \max(N,M)^{\omega}) \leq O(\max(N,M)^{2.373})$ \cite{LeGall}. To see this, note that computing the inverse of a matrix can be done in matrix multiplication time, and that these pseudoinverses can be computed by orthogonally projecting onto the complement of the all $1$s vector and then taking a usual inverse. Furthermore, the trace of the square root of a matrix is the sum of the square roots of the eigenvalues, and the eigenvalues can be computed in matrix multiplication time  \cite{eigenvalues-paper}, noting that we do not need the most computationally difficult step (c) of \cite{eigenvalues-paper}, which computes the eigenvectors. Using backpropagation, computing the gradient has the same time complexity as computing the function.


\subsection{Connections to prior work}
\label{sec-prior-work}

The closest analogue of COPT is \cite{GOT2}, which also builds on GOT. One difference is that COPT ensures that the mass of the larger graph is evenly distributed over the smaller graph, while \cite{GOT2} allow different vertices to carry different amounts of mass. \cite{GOT2} also does not use the resulting distance for sketching, like COPT does.



It is important to distinguish between graph distances defined using convex optimization, such as those defined in \cite{BentoGraphDist}, and nonconvex optimization, such as COPT. Both can be relaxations of optimization problems over permutations. For one natural loss function, \cite{lyzinski} showed that a nonconvex relaxation better approximates the optimum permutation than a convex relaxation - in fact, with high probability, the convex relaxation is not a permutation at all. While the loss function in \cite{lyzinski} is somewhat different from COPT's, we expect the same distinction between convex and nonconvex optimization to apply in our case.

However, the specific distance studied by \cite{lyzinski} is very different from COPT. That distance, which was also used by \cite{xuCarinScalable}, is formally similar to GW, but it uses the adjacency matrices directly instead of the graph metrics $d_X,d_Y$. Thus, the distance between a graph $G$ and $G$ together with one additional edge $e$ will be similar regardless of the location of $e$, while in COPT and GW it will be larger if $e$ connects two clusters that were far apart in $G$. 

The graph signal, as used in COPT, can be compared to a graph embedding. In fact, the graph signal defines an embedding of the set of vertices into a vector space of random variables. This space has coordinates given by the eigenvectors of the graph Laplacian, with the distance along the $i$ coordinate weighted by $\lambda_i^{-1/2}$ (where $\lambda_i$ is the $i$th eigenvalue). This can be compared to the graph embedding based on the first $k$ eigenvectors of the Laplacian, which is equivalent to weighting the first $k$ coordinates by $1$ and all other coordinates by $0$. \cite{NikolentzosGraphSimilarity} used this embedding to define a graph distance. Unlike COPT, this distance is not invariant under changes of coordinates and thus is discontinuous as a function of the adjacency matrix whenever an eigenvalue has multiplicity $>1$.

The first $k$ eigenvalues of the Laplacian were used to measure similarities between graphs by both \cite{liuSpectralCoarsening} and \cite{loukas}, although they were not viewed as coordinates of a graph embedding in those papers. These papers construct a sketch by searching for a smaller graph that optimizes their spectral similarity measures, subject to combinatorial restrictions on the structure of the graph sketch. COPT performs a similar search but without combinatorial restrictions.

Along similar lines but even further from COPT is \cite{LoukasVandergheynst}, which defines a sketch as a purely combinatorial process where randomly chosen edges are contracted. Still, \cite{LoukasVandergheynst} proves upper bounds on the difference between the spectra of the original graph and the sketch with high probability.

Graph embeddings that are not necessarily based on the eigenfunctions of the Laplacian have also been used to compare graphs, such as in the work of \cite{xuCarinEmbedding}, which combines arbitrary graph embeddings with Gromov-Wasserstein distance.

Sketching based on optimal transport was also used in \cite{GargJakkola}, which chooses the sketch to be a subgraph of the original graph, unlike COPT where the sketch is a new graph. Because of this, it can use the usual Wasserstein optimal transport distance as a distanct function. Like the GW distance, this approach preserves largely metric, rather than spectral, information on graphs.


\section{Experiments}\label{s-experiments}
COPT can be used both for sketching when given one graph, and finding the distance when given two graphs. Here we demonstrate its effectiveness on a variety of tasks: sketching, retrieval, classification, and summarization. Additional experiments, such as using low-dimensional COPT sketches to visualize relations between graphs, and an extensive comparison with GOT \cite{MGCP}, can be found in the supplementary material.

\subsection{Graph Sketching}
\label{sec-sketching}

We measure COPT sketching quality and compare with state of the art techniques: OTC \cite{GargJakkola}, an OT-based compression method that uses Boolean relaxations to create a compressed graph that's a \textit{subgraph} of the original graph; variation neighborhood (Variation) \cite{loukas}, a combinatorial optimization approach to graph coarsening; REC \cite{LoukasVandergheynst}, a randomized edge contraction algorithm that preserves the spectrum; algebraic distance (Algebraic) \cite{ChenSafro, RonSafroBrandt}, which contracts edges based on weights calculated using the Jacobi method; affinity (Affinity) \cite{LivneBrandt}, a vertex proximity heuristic; and heavy edge matching (HeavyE) \cite{Dhillon}, an edge contraction algorithm based on the weight of an edge and the degrees of its joining vertices.

We determine the sketching quality by measuring the graph classification accuracy on sketched graphs. Specifically, for each of the benchmark algorithms and each dataset, we first sketch the graphs by a given compression factor, then use 70\% of the sketched graphs to train an SVM with the multiscale Laplacian graph kernel \cite{multiscaleLap, kriegeBZRMD, kernelVishwanathan}, a kernel able to incorporate structural information of neighborhoods in the graph over a range of sizes. Finally we test the classification accuracy on the remaining 30\% of sketched graphs. This is done for both 2- and 4-fold compression.

This is done on four benchmark datasets over diverse domains: Proteins \cite{proteinData}, BZR\_MD \cite{kriegeBZRMD}, MSRC\_9 \cite{neumannMSRC}, and Enzymes \cite{enzymeData}. The SVM is trained with parameters found using $3$-fold cross validation on the training set, using a fast approximation of the multiscale Laplacian kernel (using the Nystr{\"{o}}m Method \cite{fastKernel}). Each accuracy measurement is repeated five times. As shown in Table~\ref{tab-acc}, COPT performs competitively across datasets for both 2- and 4-fold vertex reduction. In particular, the fact that COPT performs strongly in the 4X compression case could be due to the fact that COPT achieves a continuous, not discrete, relaxation with a weighted Laplacian. 

\begin{table}[t]
  \addtolength{\tabcolsep}{-2.pt} 
\centering
\small 
\begin{tabular}{|c|c|c|c|c|c|c|c|c|} 
\hline
 & \multicolumn{4}{|c|}{2X compression} & \multicolumn{4}{|c|}{4X compression}\\
\hline
   & BZR\_MD & MSRC\_9 & Proteins & Enzymes & BZR\_MD & MSRC\_9 & Proteins & Enzymes \\
  \hline
 OTC     & 60.7$\pm$4.0 & 80.9$\pm$4.5 & 72.8$\pm$.8 & 29.1$\pm$4.6 &64.3$\pm$2.7 & \textbf{84.8$\pm$6.7} & 66.7$\pm$1.8 & 25.2$\pm$2.7  \\ 
 HeavyE & 61.7$\pm$4.8 & 79.7$\pm$6.3 & 72.3$\pm$3.3 &27.8$\pm$2.3 & 55.0$\pm$4.7 &76.1$\pm$7.9 &  72.2$\pm$2.7 & 24.9$\pm$  1.9  \\ 
 Variation & 60.2$\pm$4.4 & 75.5$\pm$2.7 & 72.1$\pm$1.2 & 31.7$\pm$1.5 &59.3$\pm$3.2 & 78.5$\pm$3.8 & 72.4$\pm$.75 & 27.7$\pm$2.1  \\ 
 Algebraic &  57.4$\pm$5.2 & 77.0$\pm$8.9 & 70.1$\pm$2.7 & \textbf{35.1$\pm$2.3} &53.4$\pm$2.5 & 75.2$\pm$6.9 & 69.1$\pm$1.8 & 24.3$\pm$2.7  \\ 
 Affinity  & 58.5$\pm$5.0 & 80.1$\pm$3.0 & 71.2$\pm$2.5 & 25.2$\pm$1.5 &53.4$\pm$3.5 & 75.8$\pm$6.2 & 70.9$\pm$2.3 & 23.5$\pm$4.0  \\ 
 REC       & 60.9$\pm$7.3 & 82.4$\pm$1.9 & 71.1$\pm$1.5 & \textbf{34.7}$\pm$2.4 &54.5$\pm$2.7 & 77.9$\pm$3.7 & 71.5$\pm$1.0 & 28.9$\pm$1.8  \\ 
 COPT      & {\bf67.6$\pm$4.0} & {\bf86.3$\pm$1.3} & {\bf74.0$\pm$1.3} & 32.2$\pm$3.3 &{\bf68.4$\pm$5.0} & 81.2$\pm$4.8 & {\bf73.7$\pm$1.5} &  {\bf33.1$\pm$4.2} \\ \hline
\end{tabular}

\caption{Graph classification accuracy comparison with state of the art techniques on datasets across diverse domains, when the number of vertices is reduced by 2- and 4-fold. Accuracies reported in \%. }\label{tab-acc}
\vspace{-18pt}
\end{table}

\subsection{Fast Graph Retrieval}
\label{exp:retrieval}

It is often useful to reduce a set of graphs to the \textit{same number of vertices} rather than by the same compression factor (reducing to different numbers of vertices), such as for fast similarity measure between graphs using a simple $l^1$ or $l^2$ distance; or for neural network training, where batched operations are the norm and common operations such as MLP require the input dimensions be the same across samples. Uniformized dense data can also be processed more efficiently on GPUs than sparse data \cite{gpu-bottleneck}.
 
\noindent\textbf{Synthetic dataset.} We test the quality of COPT sketching to \textit{equal} number of vertices by graph retrieval quality, as judged by accuracy of the class of the nearest neighbor. 
Specifically, we take the dataset $\{G_X\}$ and queries $\{G_Q\}$ to be 600 and 180 randomly generated 50-node graphs, respectively, each evenly distributed amongst six classes: random geometric \cite{randGeoGraph}, block-2 
, block-3, block-4 \cite{blockModel}, Barabasi-Albert \cite{barabasi}, and random regular graphs \cite{randRegGraph}. We vectorize each graph $G$ in two ways: 1) sketch $G$ to 15 nodes with COPT, and flattening the upper triangular part of the sketched Laplacian to obtain a $120$-dimensional vector
, 2) take the spectral projection of $G$'s Laplacian, specifically the eigenvectors corresponding to the three smallest non-zero eigenvalues (the zero eigenvalue corresponds to the constant eigenvector), yielding a $150$-dimensional vector. Smallest eigenvalues are taken as the lower spectrum corresponds to global structure. Given a query graph vector $v_q$, we take its predicted class to be the class of its nearest neighbor, where the distance is determined with $l^1$ distance for COPT sketches, and $l^2$ for spectral projections. 


\noindent\textbf{Results.} When taking the nearest neighbor's class as the predicted class, COPT achieves $97.8\pm 1.1\%$ accuracy, which is $15.7\%$ higher on average than the spectral projections accuracy of $82.8\pm .6\%$.
To illustrate the speed advantage of equi-dimensional sketching, note that retrieval accuracy on COPT sketches trails only $2.0\%$ behind the GW accuracy of $99.8\pm .3\%$ on the \textit{original, non-reduced} graphs, while being $2000$X faster: $1.81\pm .07$ ms compared to $3.69\pm .07$ s. 
Thus COPT is ideally suited for situations where fast execution speed is critical, e.g. as a component in a pipeline.

\noindent\textbf{Combining in pipeline.} In practice, a faster, but coarser, algorithm is often used to filter out candidates for a more accurate but time-consuming method \cite{kusnerRetrieval}, so we report accuracies for pipelines that 1) does either $l^1$ retrieval on COPT sketches or $l^2$ retrieval on spectral projections to filter out unlikely candidates, and 2) runs GW on the remaining candidates. 
Retrieval using COPT sketched Laplacians significantly outperforms spectral projections of original Laplacians, 
For instance, when the fast algorithm is allowed to filter out all but the top 3 candidates, COPT+GW pipeline achieves $98.7\pm 0.3\%$ accuracy, compared with $89.4\pm 2.4\%$ for spectral projections+GW, with comparable timings as the compute bottleneck lies in the GW component of the pipeline. See appendix for the full comparison.

\noindent\textbf{Real dataset.} We also compare reduction to equal dimensions on the real dataset BZR\_MD, with on average $21.3$ nodes per graph. When reduced to $7$ nodes per graph and the upper triangular part of the reduced Laplacian taken as above, COPT achieves $57.2\pm 4.9\%$ accuracy in the class of the nearest neighbor, compared to $52.2\pm5.5 \%$ for OTC \cite{GargJakkola}. Here $100$ of the $306$ graphs in the dataset are sampled to query the remaining $206$ graphs. This is repeated $20$ times. 


\subsection{Graph Summarization}
\label{sec:graph_summary}
Figure~\ref{fig:summary} visually demonstrate that COPT preserves the most relevant global structures on graphs, across graphs of varying global structures. The sketched graph visualizations are obtained from Algorithm~\ref{alg:copt} by declaring an entry in the sketched Laplacian $L_Y$ an edge if it lies above a given threshold. The node labels on the sketched graphs are determined using the transport plan $P$, specifically the label on a node contains the two top-weighted nodes in the original graph whose mass flowed into that node. This shows that 1) COPT preserves important global graph structures, and 2) structurally similar nodes in the original graph are sketched to the same or nearby nodes. See supplementary material for more examples.

\begin{figure*}
\centering
\begin{minipage}[b]{.29\textwidth}\centering
\begin{subfigure}
         \centering
         \includegraphics[trim={2cm 1.5cm 1.9cm 1.5cm},clip,width=.97\linewidth]{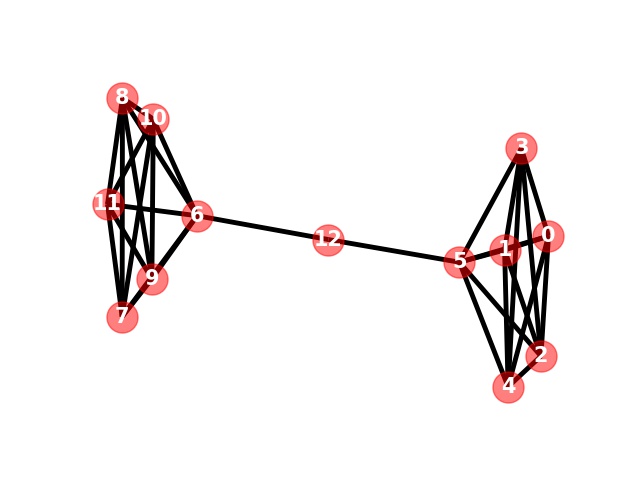}
     \end{subfigure}
     \begin{subfigure}
         \centering
         \includegraphics[trim={2cm 1.5cm 1.9cm 1.5cm},clip,width=.97\linewidth]{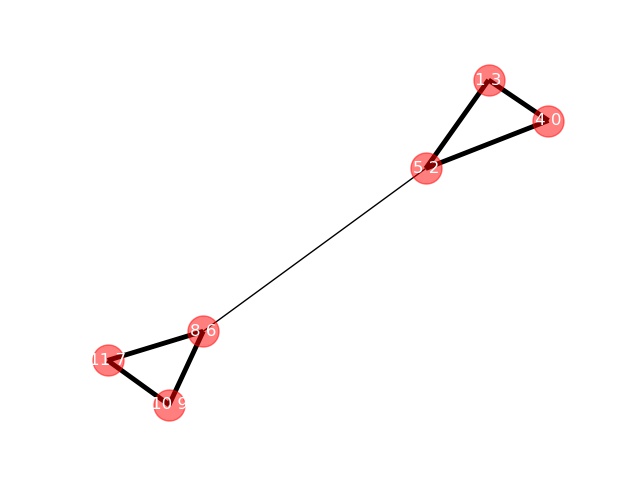}
     \end{subfigure}
\caption{Barbell.}
\end{minipage}\qquad
\begin{minipage}[b]{.29\textwidth}\centering
\begin{subfigure}
         \centering
         \includegraphics[trim={2cm 1.5cm 2cm 1.5cm},clip,width=.97\linewidth]{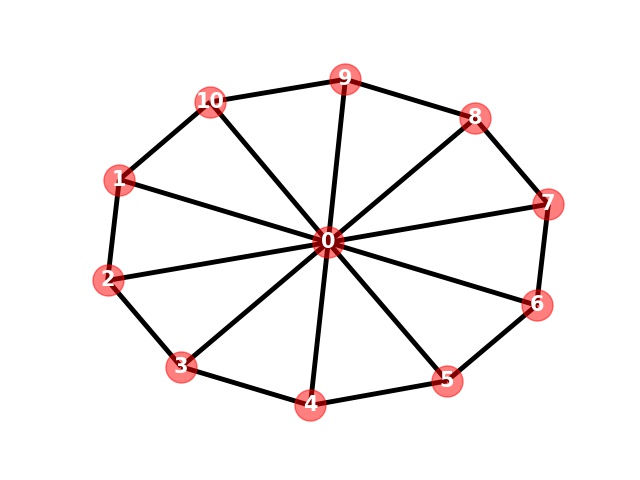}
     \end{subfigure}
     \begin{subfigure}
         \centering
         \includegraphics[trim={2cm 1.5cm 1.8cm 1.5cm},clip,width=.97\linewidth]{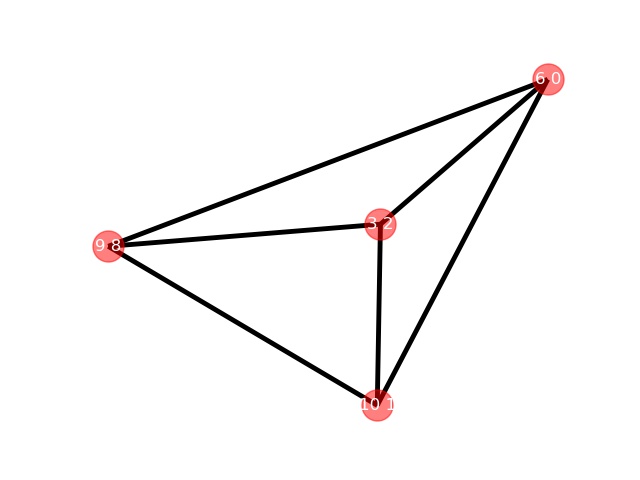}
     \end{subfigure}
\caption{Wheel.}
\end{minipage} 
\begin{minipage}[b]{.29\textwidth}\centering
\begin{subfigure}
         \centering
         \includegraphics[trim={2cm 1.5cm 1.9cm 1.5cm},clip,width=.97\linewidth]{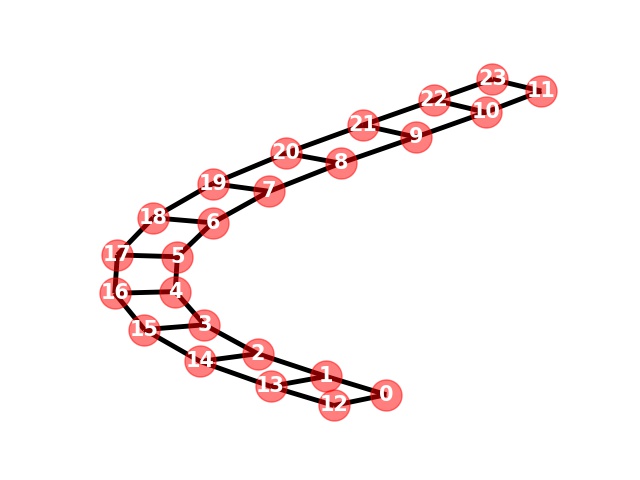}
     \end{subfigure} 
     \begin{subfigure}
         \centering
         \includegraphics[trim={2cm 1.5cm 1.9cm 1.5cm},clip,width=.97\linewidth]{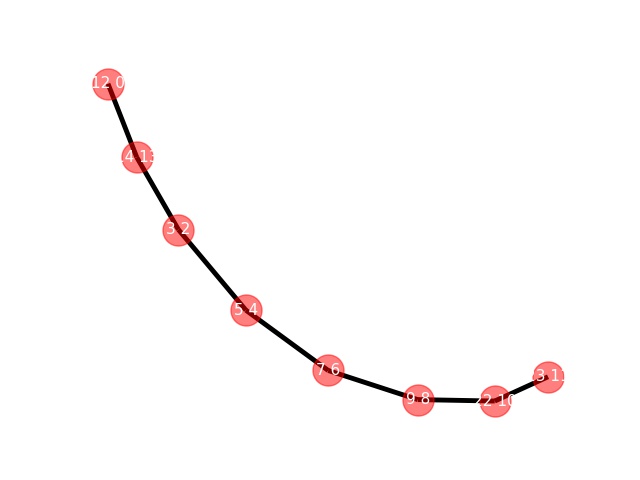}
     \end{subfigure}
     \caption{Ladder.}
\end{minipage}\qquad
\caption{Orginal graphs (top) and their sketched graphs. The node labels on the sketched graphs are determined using the transport plan $P$, specifically the label on a node contains the two top-weighted nodes in the original graph whose mass flowed into that node. COPT sketches structurally similar nodes in the original graph to the same or nearby nodes.} \label{fig:summary}
\vspace{-18pt}
\end{figure*}

\section{Acknowledgments}
This research was conducted during the period Will Sawin served as a Clay Research Fellow.



\bibliography{main}
\bibliographystyle{alpha}



\section*{Supplement Outline}
In this supplement, we give full proofs of Lemmas \ref{min-formula} and \ref{COPT-metric}, further discuss COPT optimizations during training, as well as give additional results on graph summarization and visualiation, training progress, and comparison between COPT and GOT, including both theoretical and empirical results that GOT is a special case of COPT when $N=M$.

\section{Full proofs to lemmas}
We give full proofs to Lemma~\ref{min-formula}, an analytic formula for the COPT metric, and Lemma~\ref{COPT-metric}, the statement that COPT is a metric.

\begin{lemma}[Lemma 3.1] Let $X$ and $Y$ be graphs with vertices sets of size $N$ and $M$ respectively. Then
\[ \inf_{ \substack{ T: \mathbb R^X \to \mathbb R^Y \\ T_{\#} \mu^X = \mu^Y }} \int_{\mathbb R^X} \sum_{ x\in X} \sum_{y \in Y}  (f(x) - Tf(y))^2 P(x,y)  d\mu^X (f)\] \[ =  M \tr (L_X^\dagger) + N \tr (L_Y^{\dagger})  - 2\tr (   ((L_Y^{\dagger})^{1/2} P^T L_X^{\dagger} P (L_Y^{\dagger})^{1/2})^{1/2}) \] where $P$ is the matrix with entries $P(x,y)$.
\end{lemma}

\begin{proof} Let $A$ be the map from $\mathbb R^X$ to $\mathbb R^{X\times Y}$ that sends a function $f$ on $X$ to the function $f(x) \sqrt{P(x,y)}$ on $X\times Y$. Similarly, let $B$ be the map from $\mathbb R^Y$ to $\mathbb R^{X\times Y}$ that sends a function $f$ to $f(y) \sqrt{P(x,y)}$. Then the distance between $A(f)$ and $B(T(f))$ in $\mathbb R^{X\times Y} $ is \[  \sum_{ x\in X} \sum_{y \in Y}  (f(x)\sqrt{ P(x,y)} - Tf(y) \sqrt{P(x,y)} )^2 =  \sum_{ x\in X} \sum_{y \in Y}  (f(x) - Tf(y))^2 P(x,y).\]

Thus, we can interpret this minimum as the Wasserstein distance between the pushforward $A_{\#} \mu^X$ of $\mu^X$ along $A$ and the pushforward $B_{\#} \mu^Y$ of $\mu^Y$ along $B$. Because $A$ and $B$ are linear maps, these distributions are both multivariate Gaussians with mean zero. Thus, we can use the formula for the Wasserstein distance between multivariate Gaussians with mean zero, which is expressed in terms of their covariance matrices \cite[Theorem 2.2 and Remark 4.2]{wassGaussian}. The formula is \[ \tr(  V(A_{\#}\mu^X)) + \tr( V(B_{\#}\mu^Y)) - 2 \tr \left(\left(  V(A_{\#}\mu^X)^{1/2}V(B_{\#}\mu^Y)V(A_{\#}\mu^X)^{1/2}\right)^{1/2} \right).\]

We have \[V(A_{\#}\mu^X)= A L_X^\dagger A^T \] and  \[V(B_{\#}\mu^Y)= B L_Y^\dagger B^T \] and we have by direct calculation \[A^T A = M I_N ,\]  \[ B^T B = N I_M,\] \[ B^T A= P,\] \[ A^T B = P^T,\] giving

\[ \tr(  V(A_{\#}\mu^X)) = \tr(  A L_X^\dagger A^T ) = \tr( L_X^\dagger A^T A) = M \tr ( L_X^\dagger) \]
\[ \tr(  V(B_{\#}\mu^Y)) = \tr(  B L_Y^\dagger B^T ) = \tr( L_Y^\dagger B^T B) = N \tr ( L_Y^\dagger) \]
For the last term, it is convenient to define the trace of the square root of a matrix as the sum of the square roots of its eigenvalues, so that it can be defined for more than just symmetric positive definite matrices, allowing us to write
\[  \tr \left(\left(  V(A_{\#}\mu^X)^{1/2}V(B_{\#}\mu^Y)V(A_{\#}\mu^X)^{1/2} \right)^{1/2} \right) =  \tr \left(\left(  V(A_{\#}\mu^X)V(B_{\#}\mu^Y)V\right)^{1/2} \right)\]
\[ = \tr \left(\left( A L_X^\dagger A^T B L_Y^\dagger B^T \right)^{1/2}\right) =  \tr \left(\left( B^T A L_X^\dagger A^T B L_Y^\dagger  \right)^{1/2}\right)\]
\[ = \tr \left(\left( P L_X^\dagger P^T L_Y^\dagger  \right)^{1/2}\right) =  \tr \left(\left( \left(L_Y^\dagger\right)^{1/2} P L_X^\dagger P^T \left(L_Y^\dagger\right)^{1/2}  \right)^{1/2}\right).\]

Combining these, we get exactly the stated formula. Note that, in our final formula, the matrix is again semidefinite, so we can take a canonical square root.
\end{proof}

\begin{lemma}[Lemma 3.2] $\Delta(X,Y)$ is a metric on the set of isomorphism classes of finite graphs. \end{lemma}

\begin{proof} If $X= Y$ then the distance is $0$, because then $L_X= L_Y$ and we can take $P$ to be the diagonal matrix $N I_N$.

Conversely, let us check that if the distance is zero then $X=Y$. Because the distance is the minimum of a continuous function over the compact set of possible values of $P$, if the distance is zero then zero is obtained for a particular value of $P$. Because the Wasserstein distance between two distributions is zero if and only if they are the same distribution, it follows that (in the notation of the previous lemma) $A L_X^\dagger A^T = B L_Y^\dagger B^T$, which in concrete terms means that \[ \left( L_X^{\dagger}\right)_{x_1x_2} = \left( L_Y^\dagger\right)_{y_1y_2}\] whenever $P(x_1,y_1)\neq 0$ and $P(x_2,y_2)\neq 0$. Taking $y_1=y_2=y$, this implies \[ \left( L_X^{\dagger}\right)_{x_1x_2} = \left( L_X^{\dagger}\right)_{x_1x_1} \] whenever $P(x_1,y)\neq 0$ and $P(x_2,y)\neq 0$. This identity implies $x_1=x_2$, because the unique maximum of a column of $L_{X}^{\dagger}$ is on the diagonal.  Thus, we see that $P(x,y)\neq 0$ for a unique $x$ for each $y$. By symmetry this must also be true for a unique $y$ for each $x$. Because $P(x,y)\neq 0$ for at least one $y$ for each $x$, and at least one $x$ for each $y$, the set where $P(x,y)\neq 0$ defines a permutation. After applying that permutation, our identity can be written \[ L_X^\dagger = L_Y^\dagger\] which implies \[ L_X =L_Y\] and thus \[X=Y.\]

Symmetry is easiest to check using Lemma \ref{min-formula} and the fact that the trace of the square root of a matrix, like the trace of a matrix, is invariant under cyclic permutations, so 
\[ \tr (   ((L_Y^{\dagger})^{1/2} P^T L_X^{\dagger} P (L_Y^{\dagger})^{1/2})^{1/2}) = \tr (   ((L_Y^{\dagger})^{1/2} P^T ( L_X^{\dagger})^{1/2} ( L_X^{\dagger})^{1/2}  P (L_Y^{\dagger})^{1/2})^{1/2})\]
\[ = \tr (   (( L_X^{\dagger})^{1/2}  P (L_Y^{\dagger})^{1/2}(L_Y^{\dagger})^{1/2} P^T ( L_X^{\dagger})^{1/2} )^{1/2}) = \tr (   (( L_X^{\dagger})^{1/2}  P L_Y^\dagger  P^T ( L_X^{\dagger})^{1/2} )^{1/2} )\]
so swapping $P$ and $P^T$, our formulas for the two distances are equal. 

Finally, we check the triangle inequality. Let $X,Y,$ and $Z$ be graphs with $|X|=N, |Y|=M,$ and $|Z|= L$. Let $P$ and $T$ satisfy the conditions in the definition of $\Delta(X,Y)$ so that the integral is within $\epsilon$ of the minimum value. Similarly let $Q$ and $S$ satisfy the conditions in the definition of $\Delta(Y,Z)$. It suffices to show that
\begin{equation}\label{desired-triangle-inequality} \left(\frac{1}{MN} \int_{\mathbb R^X} \sum_{ x\in X} \sum_{y \in Y}  (f(x) - Tf(y))^2 P(x,y)  d\mu^X (f)\right)^{1/2} \end{equation}\[ +  \left(\frac{1}{NL} \int_{\mathbb R^Y} \sum_{y\in Y} \sum_{z \in Z}  (f(y) - Sf(z))^2 Q(y,z)  d\mu^Y (f)\right)^{1/2} \] \[ \geq  \left(\frac{1}{ML} \int_{\mathbb R^X} \sum_{ x\in X} \sum_{z \in Z}  (f(x) - STf(z))^2  \sum_{y\in Y} \frac{P(x,y) Q(y,z)}{N}  d\mu^X (f)\right)^{1/2}.\]
because then $\sum_{y\in Y} \frac{P(x,y) Q(y,z)}{N}$ and $S \circ T$ satisfy the conditions in the definition of $\Delta(X,Z)$, and then taking $\epsilon$ sufficiently small, we obtain the triangle inequality. 

This inequality follows from the Cauchy-Schwarz inequality
\begin{equation}  \left(\frac{1}{MNL} \int_{\mathbb R^X} \sum_{ x\in X}\sum_{y\in Y}  \sum_{z \in Z}  (f(x) - STf(z))^2  P(x,y) Q(y,z)  d\mu^X (f)\right)^{1/2} \leq\end{equation}
\begin{equation}\label{first-term} \left(\frac{1}{MNL} \int_{\mathbb R^X} \sum_{ x\in X}\sum_{y\in Y}  \sum_{z \in Z}  (f(x) - Tf(y))^2  P(x,y) Q(y,z)  d\mu^X (f)\right)^{1/2} \end{equation} \begin{equation}\label{second-term} +  \left(\frac{1}{MNL} \int_{\mathbb R^X} \sum_{ x\in X}\sum_{y\in Y}  \sum_{z \in Z}  (Tf(y) - STf(z))^2  P(x,y) Q(y,z)  d\mu^X (f)\right)^{1/2}\end{equation}
where \eqref{first-term} simplifies to the first term of \eqref{desired-triangle-inequality} by using $\sum_{z\in Z}Q(y,z)=L$ and \eqref{second-term} simplifies to the second term of \eqref{desired-triangle-inequality} by using $\sum_{x\in X} P(x,y) = N$ and $T_{\#}\mu_X =\mu_Y$ so the integral against $\mu_X$ of a function of $Tf$ is equal to the integral against $\mu_Y$ of the same function of $f$.

\end{proof}

\section{Further details on COPT optimization}
\label{sec:optimization}

Here we elaborate further on COPT's optimization routine. As the objective Equation~\ref{min-formula} is not globally convex, gradient descent can fall into local minima. To facilitate faster convergence and convergence towards global minima, we use a combination of learning rate scheduling and learning rate hikes. 

Specifically, the learning rate is initialized at $0.4$. During training, it's scaled multiplicatively by $0.7$ per $100$ iterations. Once the \emph{change} in loss drops below a threshold, set at $0.002$ throughout, for 10 different iterations, the learning rate is increased five-fold, capped at 4.0. If an upper bound on the number of iterations is set, learning rate hikes stop $200$ iterations before the max number of iterations, to allow convergence. These parameters can be tuned with respect to the downstream task, we have found COPT to be robust with respect to the parametrizations. 

Figure~\ref{fig:lr_hike} illustrates the effects learning rate hikes on training loss, where LR hiking is the only difference between the two curves. The loss increases briefly when the learning rate is hiked, but can soon drop below the level had there been no hikes. LR hiking not only improves convergence rate, but experiments show it also changes the finaly transport plan $P$, indicating convergence to a better minimum.

We observe that this optimization routine allows for \textbf{faster runtime} per iteration than stochastic exploration used in GOT, which, to avoid converging to local minima, uses a stochastic exploration method that minimizes the \emph{expectation} of the distance, rather than the distance itself. But this requires a nontrivial number (e.g. 10) of random explorations at each training step to achieve good performance.

For instance, when aligning the same set of $50$-node graphs, on the same machine and CPU, $1000$ iterations of GOT takes $18.8 \pm 0.72 $ seconds, and COPT takes $3.18 \pm 0.59 $ seconds, with settings that were tuned to produce the best community discovery results for each method, which as observed are commensurable (as described in \S\ref{sec:compare}). This is repeated 20 times.


\begin{figure}
    \centering
    \includegraphics[width=.5\linewidth]{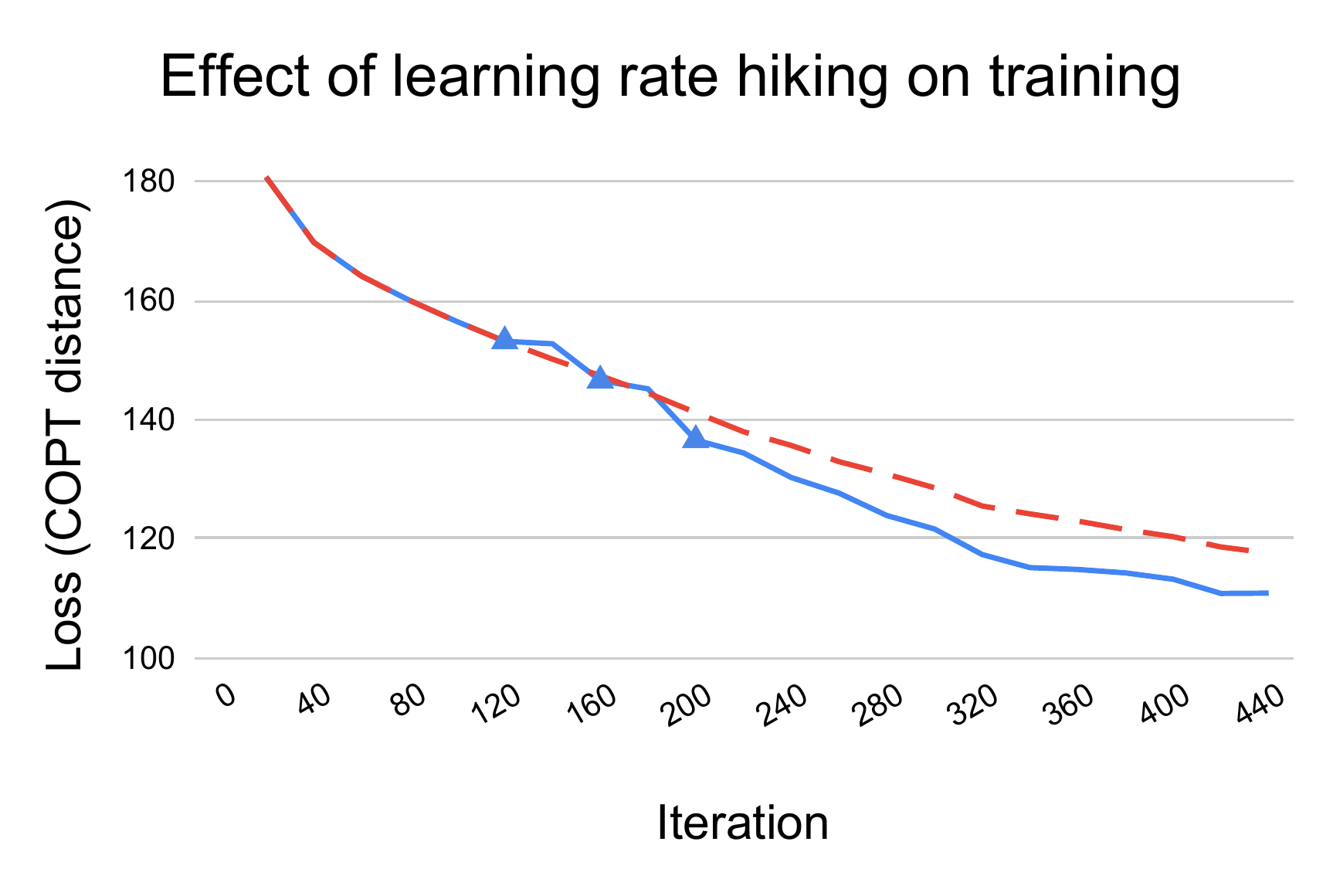}
    \caption{The impact of learning rate (LR) hiking on COPT training, in this case computing the alignment and distance between two 25-node graphs. The dashed line represents no LR hiking, the solid line represents LR hiking, the triangles indicate LR hikes. The hikes are triggered when the change in loss falls below $0.002$ over $10$ different iterations. Not only are the rates of convergence different, the presence or absence of LR hiking can lead to different permutations.}\label{fig:lr_hike}
\end{figure}

\begin{wrapfigure}{r}{0.45\textwidth}
    \vspace{-15pt}
    \begin{center}
    \includegraphics[width=.99\linewidth]{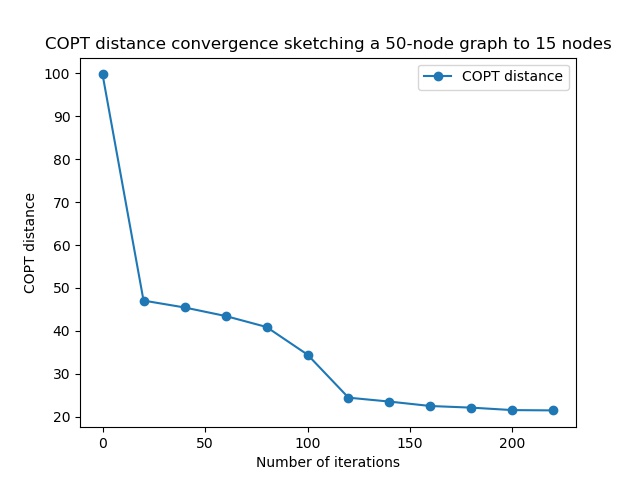}
    \end{center}
    \vspace{-10pt}
    \caption{Sketching a $50$-node graph down to $15$ nodes with COPT, which typically converges in $\sim 200$ iterations.}
    \vspace{-14pt}
    \label{fig:train_convergence}
\end{wrapfigure}



\section{Further comparison with GOT}
\label{sec:compare}
Under the special case $N=M$, beacuse the transport plan $P$ converges to a permutation, COPT reduces to the GOT metric described in \cite{MGCP}. We prove this in Lemma~\ref{to-permutation}. As this case uses the same optimization backbone when $N\ne M$ as well as for sketching, COPT and GOT differ in their implementations. 

We want to compare the methods in the special case $N=M$, to ensure that, as in theory, the two implementations are commensurable. 

However, the permutations between the two in the alignment task are not directly comparable, as many graphs carry symmetries that allow multiple transport plans to achieve the same objective minimum \footnote{Inspection by hand reveals that, on ``simple" graphs such as two-block graphs, the transport plans produced by COPT and GOT are the same up to symmetry.}.


Thus, to compare our implementation with GOT's, we repeat one key alignment experiment in \cite{MGCP}, where $40$-node four-community graphs are aligned with corrupted and permuted versions of themselves, and the alignment quality is measured using the normalized mutual information (NMI).

Specifically, we 1) randomly delete a given number of nodes, 2) permute nodes in the reduced graphs, 3) align between the original graphs and the permuted reduced graphs using COPT and GOT, and 4) compare the resulting NMI using the permutations produced by the alignment algorithms. 

\begin{table}[t]
\addtolength{\tabcolsep}{-3.5pt}
\centering
\begin{tabular}{lclclclclclclclclc}
\hline
\hline
 \#edge rem. & 30 & 60 & 90 & 120 & 150 & 180 & 210 & 240 \\
 \hline
GOT & .997$\pm$.013 & 1.$\pm$0. & .991$\pm$.022 & .960$\pm$.059 & .837$\pm$.092 & .855$\pm$.10 &  .844 $\pm$.096   &   .862$\pm$.11   \\
 COPT & 1.$\pm$0. & 1.$\pm$0. &  .994$\pm$.026 & .964$\pm$.094 & .873$\pm$.13 & .848$\pm$.13  & .897$\pm$.125 &  .871$\pm$.14  \\
 \hline
\end{tabular}
\vspace{2 pt}
\caption{\textbf{Alignment quality as gauged by community discovery performance.} Normalized mutual information (NMI) scores between $40$-node graphs against corrupted and then permuted versions of themselves, using GOT and COPT implementations. Higher is better. First row indicates number of edges removed. The same upper bound on the number of iterations, 1000, is used for both methods, and parameters tuned.
GOT is a special case of COPT when $N=M$, and the NMI above demonstrates that they are also commensurable empirically.}\label{tab:mi}
\vspace{-10pt}
\end{table}

As shown in Table~\ref{tab:mi}, the NMI are commensurable. Each measurement is repeated 20 times after parameter tuning. Furthermore, the GOT NMI are consistent with \cite{MGCP}, which also report that these NMI scores outperform the NMI for the same evaluation settings using GW.
\vspace{3 pt}

\section{Characterization of minima} The following lemma characterizes the minima in the COPT objective function, in general and in the special case $N=M$.

\begin{lemma}\label{to-permutation} Any local, hence global, minimum of the function  \[\inf_{ \substack{ T: \mathbb R^X \to \mathbb R^Y \\ T_{\#} \mu^X = \mu^Y }} \int_{\mathbb R^X} \sum_{ x\in X} \sum_{y \in Y}  (f(x) - Tf(y))^2 P(x,y)  d\mu^X (f)\] on the convex polyhedron
\[  \left \{  P: X\times Y \to \mathbb R^{+} \mid \sum_{x \in X} P (x,y) = N, \sum_{y \in Y} P(x,y) = M \right \} \] that is isolated, in the sense that no other point within distance $\epsilon$ for any $\epsilon$ is a local minimum, is a vertex of that polyhedron.

In particular, if $N=M$, every isolated local minimum is a permutation matrix. \end{lemma}

Because local minima being isolated is a generic condition, this shows that the local minima are permutations away from a set of Laplacians $L_X, L_Y$ with measure $0$, and thus COPT will converge to a permutation outside a set of measure $0$.

\begin{proof} Let $P_0$ be a local minimum of \[ \inf_{ \substack{ T: \mathbb R^X \to \mathbb R^Y \\ T_{\#} \mu^X = \mu^Y }} \int_{\mathbb R^X} \sum_{ x\in X} \sum_{y \in Y}  (f(x) - Tf(y))^2 P(x,y)  d\mu^X (f) = \min_{\gamma} L (P, \gamma) \]
 where the minimum is taken over measures $\gamma$ on $\mathbb R^X \times \mathbb R^Y$ whose projection to $\mathbb R^X$  is $\mu^X$ and whose projection to $\mathbb R^Y$ is $\mu^Y$, and where\[ L(P, \gamma) =\int_{\mathbb R^X} \int_{ \mathbb R^y} \sum_{ x\in X} \sum_{y \in Y}  (f(x) - g(y))^2 P(x,y)  d\gamma (f,g).\]
 We switch from the transport map to the joint measure formulation of optimal transport as only for the joint measure is the minimum always attained.

Let $\gamma_0$ be the value of $\gamma$ which attains this minimum. Thus $P_0$ is also a local minimum of $L(P, \gamma_0)$. To prove this, note that if $L(P, \gamma_0)$ a smaller value at some nearby point, $\min_{\gamma} L(P, \gamma)$ a smaller value, because $\gamma_0$ is included in the minimum over $\gamma$, and so $P_0$ will fail to be a local minimum of $\min_{\gamma} L(P, \gamma)$ assumed.

Because $L(P, \gamma_0)$ is a linear function of $P$, its local minima are simply faces of the polyhedron of possible values of $P$. If $P$ is not a vertex, then it lies in a positive-dimensional face, on which $L(P, \gamma_0)$ is constant, and thus on which every value of $\min_{\gamma} L(P,\gamma)$ is at most $L(P_0,\gamma_0) =\min_{\gamma} L(P, \gamma)$. In this situation, the only way $P_0$ can be a local minimum of $\min_{\gamma} L(P, \gamma)$ is if $\min_{\gamma} L(P, \gamma)$ is actually constant on a neighborhood of $P_0$ in this face. In that case, every point in the constant region is a local minimum, contradicting our assumption that $P$ is a local minimum.

Finally, we must check that every vertex if $N=M$ corresponds to a permutation. To see this, let $P$ be a vertex, and consider a bipartite graph with vertices $X \cup Y$ with an edge connecting $x\in X$ and $y\in Y$ if and only if $P(x,y) \neq 0$.

Let us check that this graph does not contain any cycle. If it did, because the graph is bipartite, we could 2-color each edge of the cycle. Then for any sufficiently small $\epsilon>0$, raising $P(x,y)$ by $\epsilon$ for $(x,y)$ each red edge and lowering $P(x,y)$ by $\epsilon$ for $(x,y)$ each blue edge would preserve all the conditions on $P$ defining the polyhedron. Similarly, lowering $P(x,y)$ for the red edges and raising $P(x,y)$ for the blue edges would preserve the condition. Thus $P(x,y)$ is a convex combination of two different points in the polyhedron and hence is not a vertex.

Because this graph does not contain a cycle, all its connected components must be trees, and every tree contains a leaf. But if $x\in X$ is a leaf, then $P(x,y)>0$ for a unique $y$, so $P(x,y)= M =N$ and thus $P(x',y)=0$ for a unique $x'$. The same holds if $y\in Y$ is a leaf. Thus every leaf is connected to another leaf and so every tree is a single edge, so in fact every component of the graph is a single edge. Thus, each $x \in X$ is connected by exactly one edge to a unique $y\in Y$, and vice versa, defining a permutation $\sigma: X\to Y$ with $P(x,y) =\begin{cases} N & y=\sigma(x) \\ 0 & y \neq \sigma(x) \end{cases}$.

\end{proof} 

\section{Further details on datasets}

Table~\ref{tab-data} contains additional details on the graph datasets used in experiments, which recall are Proteins \cite{proteinData}, BZR\_MD \cite{kriegeBZRMD}, MSRC\_9 \cite{neumannMSRC}, and Enzymes \cite{enzymeData}. 

\begin{table}[ht]
\addtolength{\tabcolsep}{-2.pt} 
\centering
\vspace{-3pt}
\begin{tabular}{lclclclclclclc}
\hline
\hline
  & \# graphs & \# classes & Avg \# nodes & Avg \# edges  \\ \hline
  Proteins & 1113 & 2 & 39.06 & 72.82  \\
  BZR\_MD & 306 & 2 & 21.30 & 225.06  \\
  MSRC\_9 & 221 & 8 & 40.58 & 97.94  \\
  Enzymes & 600 & 6 & 32.63 & 62.14  \\
  \hline
\end{tabular}
\caption{Details on experimental datasets. }
\label{tab-data}
\vspace{-4pt}
\end{table}


Note that we work with connected, undirected graphs. The ``connected" assumption is unnecessary, but we focus on connected graphs as they are more relevant for most applications. The ``undirected" assumption is necessary because we define the variance-covariance matrix of $\mu^X$ as the pseudoinverse of the Laplacian. A variance-covariance matrix is always symmetric, so the Laplacian must be symmetric. 

\section{Further experiments}

This section contains further experiments to show the effectiveness and utility of COPT for visualization and summarization tasks.

\subsection{Graph projection to low dimensions}\begin{figure*}
\centering
\begin{minipage}[b]{.45\textwidth}
\begin{subfigure}
         \centering
         \includegraphics[trim={2cm 1.5cm 1.9cm 1.1cm},clip,width=\linewidth]{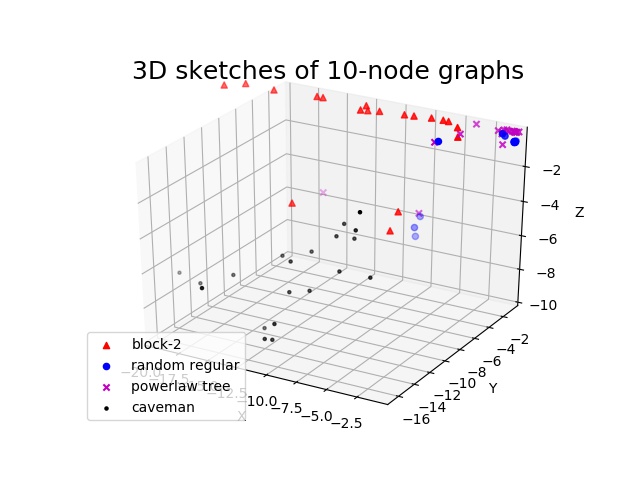}
     \end{subfigure}
     \end{minipage}\qquad
\begin{minipage}[b]{.45\textwidth}     
     \begin{subfigure}
         \centering
         \includegraphics[trim={2cm 1.5cm 1.9cm 1.1cm},clip,width=\linewidth]{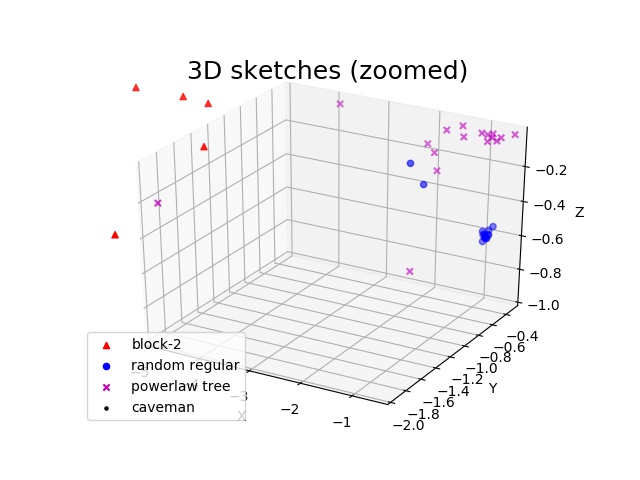}
     \end{subfigure}
\end{minipage}\qquad
\caption{3D sketches of 80 ten-node graphs (left) and zoomed in version (right) reveal that sketches of the same class cluster together or lie in the same subspace.}\label{label-a} \label{fig:projection}
\end{figure*}

Projecting graphs to low dimensions can be an effective technique for visualizing relations amongst graphs, analogous to how projection techniques such as t-SNE \cite{t-SNE} can yield important insights in the distribution of data points.

Here a set of $80$ ten-node graphs consisting evenly of four classes of graphs \footnote{These classes were chosen as they admit meaningfully distinct properties on ten nodes, unlike for instance a four-community graph.}: 2-block \cite{blockModel}, random regular \cite{randRegGraph}, powerlaw tree, and caveman \cite{cavemanGraph}, are sketched down to graphs on three nodes, and the three entries in the strict upper triangular part of the sketched Laplacian are \emph{canonicalized} by sorting, and used as three-dimensional vector representations of the graphs.

As seen in Figure~\ref{fig:projection} and the zoomed in version, the 3D sketches of each class of graphs roughly follow the same trend, whether clustering together or lying in the same subspace. 

\subsection{Graph retrieval}


Following up on \S\ref{exp:retrieval}, this subsection provides a more detailed comparison between the retrieval quality of different pipelines, namely retrieval using [COPT sketches + GW] and [spectral projections + GW], when both COPT sketching and spectral projections are allowed to reduce the original graphs to a \textit{fixed} number of vertices. As seen in Figure~\ref{fig:search_acc}, retrieval using COPT sketched Laplacians outperforms spectral projections of original Laplacians by a margin in accuracy, where accuracy is determined by whether the top-retrieved candidate has the same class as the query. 


\begin{figure}[!htb]
\centering
  \includegraphics[width=200pt]{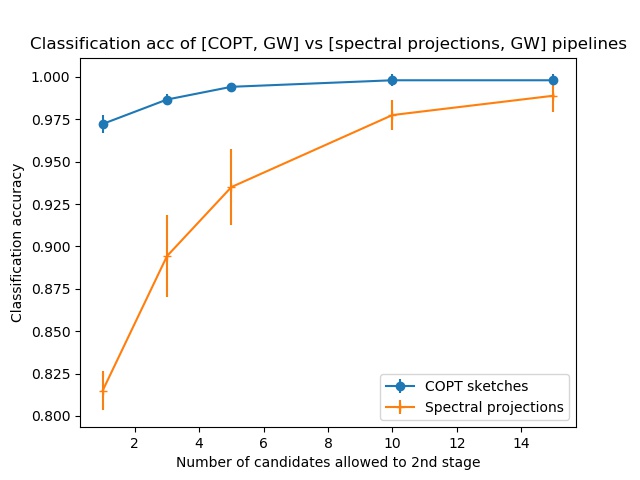}
  \vspace{-10pt}
  \caption{Comparing nearest-neighbor based classification between pipelines: [COPT sketches+GW] vs. [spectral projections+GW]. The x-axis indicates the number of candidates from the coarse method ($l^1$ retrieval on COPT sketches or $l^2$ retrieval on spectral projections) allowed to advance to the finer but more time-consuming method (GW).
  The query set and search dataset consist of 180 and 600 50-node graphs across six categories, respectively. 
  }\label{fig:search_acc}
\end{figure}

\subsection{Graph summarization examples}
We provide additional examples for graph summarization in Figure~\ref{fig:summary2}, in which graphs of varying initial structures are sketched down to a reduced number of nodes. As seen in these sketches, COPT is able to preserve the initial global structures of graphs, sketching structurally adjacent nodes in the original graph to the same node or nearby nodes in the reduced graph.

\begin{figure*}[h]
\centering
\begin{minipage}[b]{.4\textwidth}\centering
\begin{subfigure}
         \centering
         \includegraphics[trim={2cm 1.5cm 1.9cm 1.5cm},clip,width=.95\linewidth]{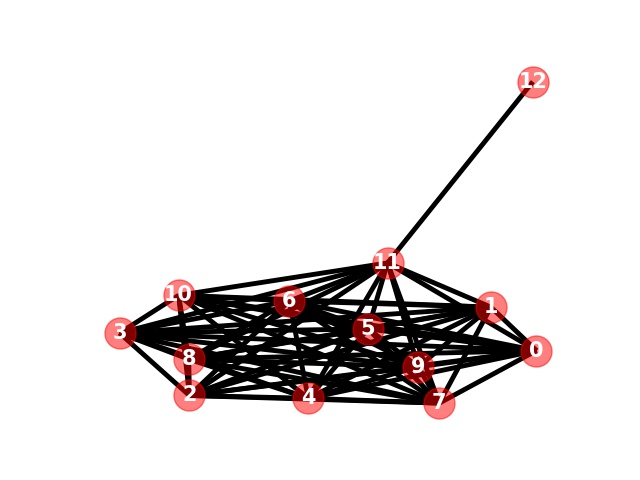}
       
     \end{subfigure}
     \begin{subfigure}
         \centering
         \includegraphics[trim={2cm 1.5cm 1.9cm 1.5cm},clip,width=.97\linewidth]{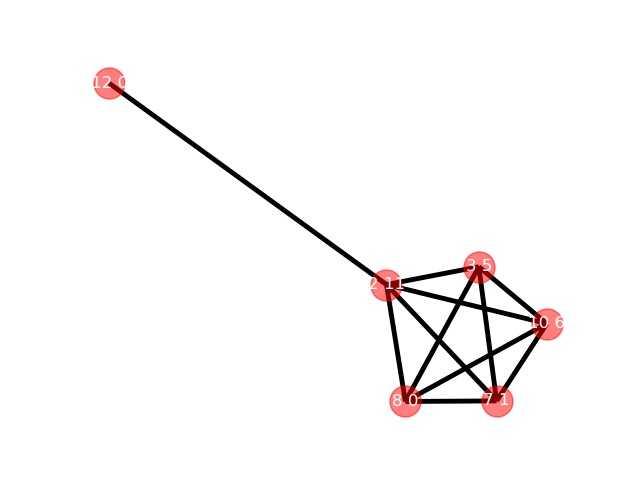}
     \end{subfigure}
\caption{Lollipop.}
\end{minipage}\qquad
\begin{minipage}[b]{.4\textwidth}\centering
\begin{subfigure}
         \centering
         \includegraphics[trim={2cm 1.5cm 1.6cm 1.5cm},clip,width=.95\linewidth]{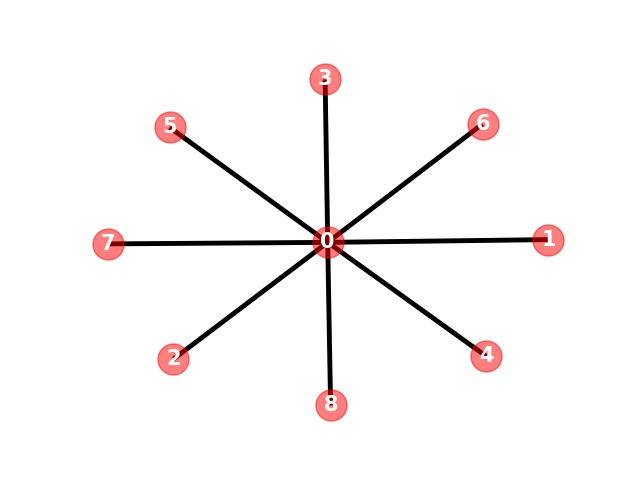}
     \end{subfigure}
     \begin{subfigure}
         \centering
         \includegraphics[trim={2cm 1.5cm 1.8cm 1.5cm},clip,width=.97\linewidth]{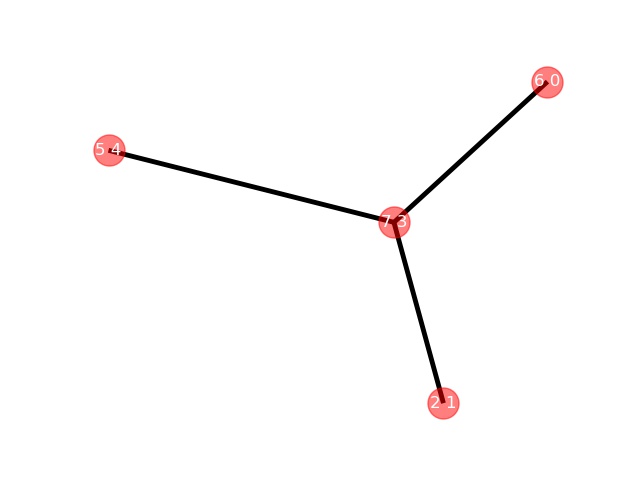}
     \end{subfigure}
\caption{Star.}
\end{minipage} 
\begin{minipage}[b]{.4\textwidth}\centering
\begin{subfigure}
         \centering
         \includegraphics[trim={2cm 1.5cm 1.9cm 1.5cm},clip,width=.95\linewidth]{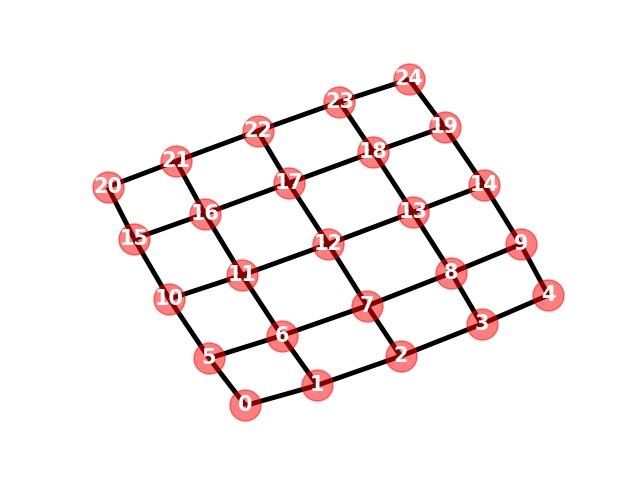}
     \end{subfigure} 
     \begin{subfigure}
         \centering
         \includegraphics[trim={2cm 1.5cm 1.9cm 1.5cm},clip,width=.97\linewidth]{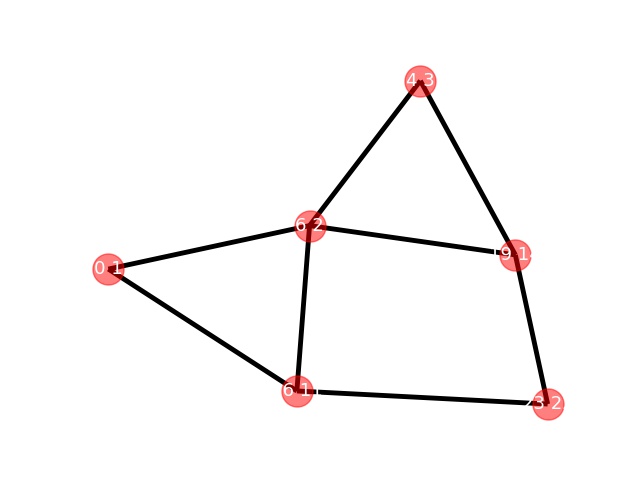}
     \end{subfigure}
     \caption{Grid.}
\end{minipage}
\begin{minipage}[b]{.4\textwidth}\centering
\begin{subfigure}
         \centering
         \includegraphics[trim={2cm 1.5cm 1.9cm 1.5cm},clip,width=.95\linewidth]{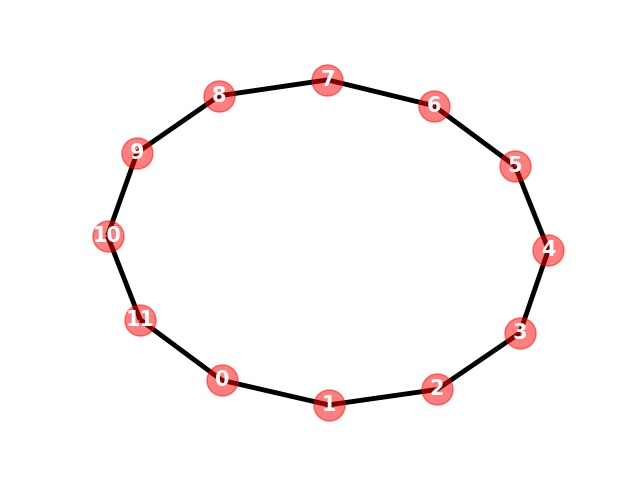}
     \end{subfigure} 
     \begin{subfigure}
         \centering
         \includegraphics[trim={2cm 1.5cm 1.9cm 1.5cm},clip,width=.97\linewidth]{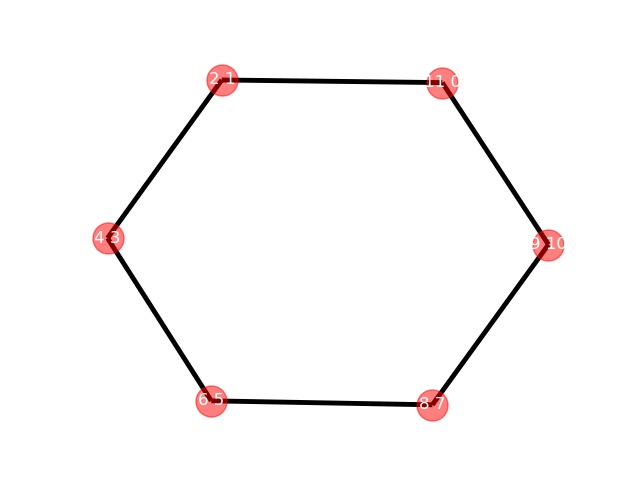}
     \end{subfigure}
     \caption{Ring.}
\end{minipage}\qquad
\caption{\textbf{Graph summarization:} orginal graphs (above) and their sketched graphs (below). The node labels on the sketched graphs are determined using the transport plan $P$, specifically the label on a node contains the two top-weighted nodes in the original graph whose mass flowed into that node. COPT sketches structurally similar nodes in the original graph to the same or nearby nodes.} \label{fig:summary2}
\end{figure*}


\end{document}